\documentclass[times,twocolumn,final]{elsarticle}

\usepackage{framed,multirow}

\usepackage{amssymb}
\usepackage{comment}
\usepackage{amsmath}

\usepackage{caption}
\usepackage{mathtools}
\usepackage{url}
\usepackage{xcolor}
\usepackage{soul} 
\definecolor{newcolor}{rgb}{.8,.349,.1}

\usepackage{rotating} 

\usepackage[switch,pagewise]{lineno} 

\usepackage[english]{babel}
\usepackage{amsthm}
\usepackage{hyperref}

\theoremstyle{definition}

\newtheorem{theorem}{Theorem}[section]

\newtheorem{lemma}[theorem]{Lemma}

\begin{document}


\begin{frontmatter}

\title{Learning Significant Persistent Homology Features for 3D Shape Understanding}%

 \author{Prachi Kudeshia}
\emailauthor{prachi.kudeshia@smu.ca}{P. Kudeshia}
\emailauthor{jiju.poovvancheri@smu.ca}{J. Poovvancheri}

\author{Jiju Poovvancheri}

\address{Saint Mary's University, Halifax, Canada}


\begin{abstract}
Geometry and topology constitute complementary descriptors of three-dimensional shape, yet existing benchmark datasets primarily capture geometric information while neglecting topological structure. This work addresses this limitation by introducing topologically-enriched versions of ModelNet40 and ShapeNet, where each point cloud is augmented with its corresponding persistent homology features. These benchmarks with the topological signatures establish a foundation for unified geometry-topology learning and enable systematic evaluation of topology-aware deep learning architectures for 3D shape analysis. Building on this foundation, we propose a deep learning-based significant persistent point selection method, \textit{TopoGAT}, that learns to identify the most informative topological features directly from input data and the corresponding topological signatures, circumventing the limitations of hand-crafted statistical selection criteria. A comparative study verifies the superiority of the proposed method over traditional statistical approaches in terms of stability and discriminative power. Integrating the selected significant persistent points into standard point cloud classification and part-segmentation pipelines yields improvements in both classification accuracy and segmentation metrics. The presented topologically-enriched datasets, coupled with our learnable significant feature selection approach, enable the broader integration of persistent homology into the practical deep learning workflows for 3D point cloud analysis.

\end{abstract}


\end{frontmatter}


\section{Introduction}

Shape understanding is central to many computer graphics and computer vision applications, such as robotic perception, Augmented and Virtual Reality, architecture, and entertainment. 
Point clouds have emerged as a fundamental representation for capturing the 3D structure of real-world environments that offer rich geometric and spatial information; however, their irregular and unordered nature poses unique challenges for machine learning models. Pioneering works like PointNet \cite{qi2017pointnet} and subsequent advancements address these inherent complexities and enable efficient learning of point clouds, which is critical for understanding object shape in 3D space. These methods encode rich geometric information from the spatial distribution of points that allows models to learn features such as local curvature, surface normals, and global shape descriptors. While deep learning methods have made remarkable progress by leveraging these geometric properties, they often struggle to capture higher-level structural relationships and global shape invariants, especially in sparse or noisy data. This has led to growing interest in topological approaches, which offer complementary insights by analyzing the intrinsic shape and connectivity of data beyond its geometric embedding. 

While geometric features enrich point-cloud understanding by extending beyond raw spatial coordinates, topological features provide a complementary and often more robust perspective by focusing on the inherent connectivity and global structure of the point cloud, irrespective of its specific geometric embedding. While geometric features might be sensitive to noise, sampling density variations, and local deformations, topological features such as the number of connected components, loops, and voids, along with their persistence across different scales, offer a more abstract and resilient representation of the underlying shape. Thus, topological features capture the fundamental shape in terms of its global arrangement and relationships between parts, providing insights into properties that remain invariant under continuous deformations, offering a powerful tool for distinguishing objects and understanding scenes beyond purely geometric considerations. Therefore, in this work, we use the topology of a point cloud to provide complementary insights on the intrinsic structure and connectivity of the data. Combined with geometric features, these topological descriptors yield more robust and informative representations for downstream 3D tasks.

Unlike geometric features, which can be directly derived from point clouds, topological features require specialized algorithms such as persistent homology to extract multi-scale connectivity patterns and higher-order relationships within the data. In persistent homology, the topological features such as connected components, loops, and voids of a dataset are studied across multiple spatial resolutions. This computationally intensive process starts by building a filtration, in which a sequence of nested simplicial complexes, such as Vietoris-Rips or Alpha complexes \cite{dey2022computational}, are constructed from a point cloud. As the scale parameter (distance threshold) increases, more simplices (edges, triangles, tetrahedra, etc.) are introduced. The filtration step dominates the computational cost, as it requires generating and updating simplices across multiple scales. Consequently, filtration demands substantial time and memory resources \cite{malott2022survey, peek2023synthetic, coskunuzer2024topological}. As a result, topological feature generation during the model training is usually not feasible, especially for large point cloud datasets. To address this limitation, we provide a topological dataset composed of persistence diagrams (PDs) computed in homology dimensions 1 and 2 (i.e., $H_{1}$ and $H_{2}$) of benchmark point cloud datasets, ModelNet40 \cite{wu20153d} and Shapenet \cite{shapenet2015}, which can be directly used during training to extract and aggregate topological features for classification and part-segmentation.

The complex filtration process across different scales results in different sizes of the persistence diagram for each point cloud. In addition, noisy topological features in persistence diagrams result in the large overall size of the persistence diagrams. Therefore, the direct utilization of variable-sized as well as large persistence diagrams of a dataset in a neural network architecture is challenging \cite{coskunuzer2024topological}. A viable solution to this problem is the filtering of noisy features from the persistence diagram of each homology dimension, such that persistence diagrams have a uniform fixed size across samples. However, recognition of such topological noise is not straightforward. In general, topological features with short persistence are considered insignificant \cite{coskunuzer2024topological}.
 \begin{figure}[h]
\centering
\hspace*{-0.5cm} 
\includegraphics[width=9.0 cm]{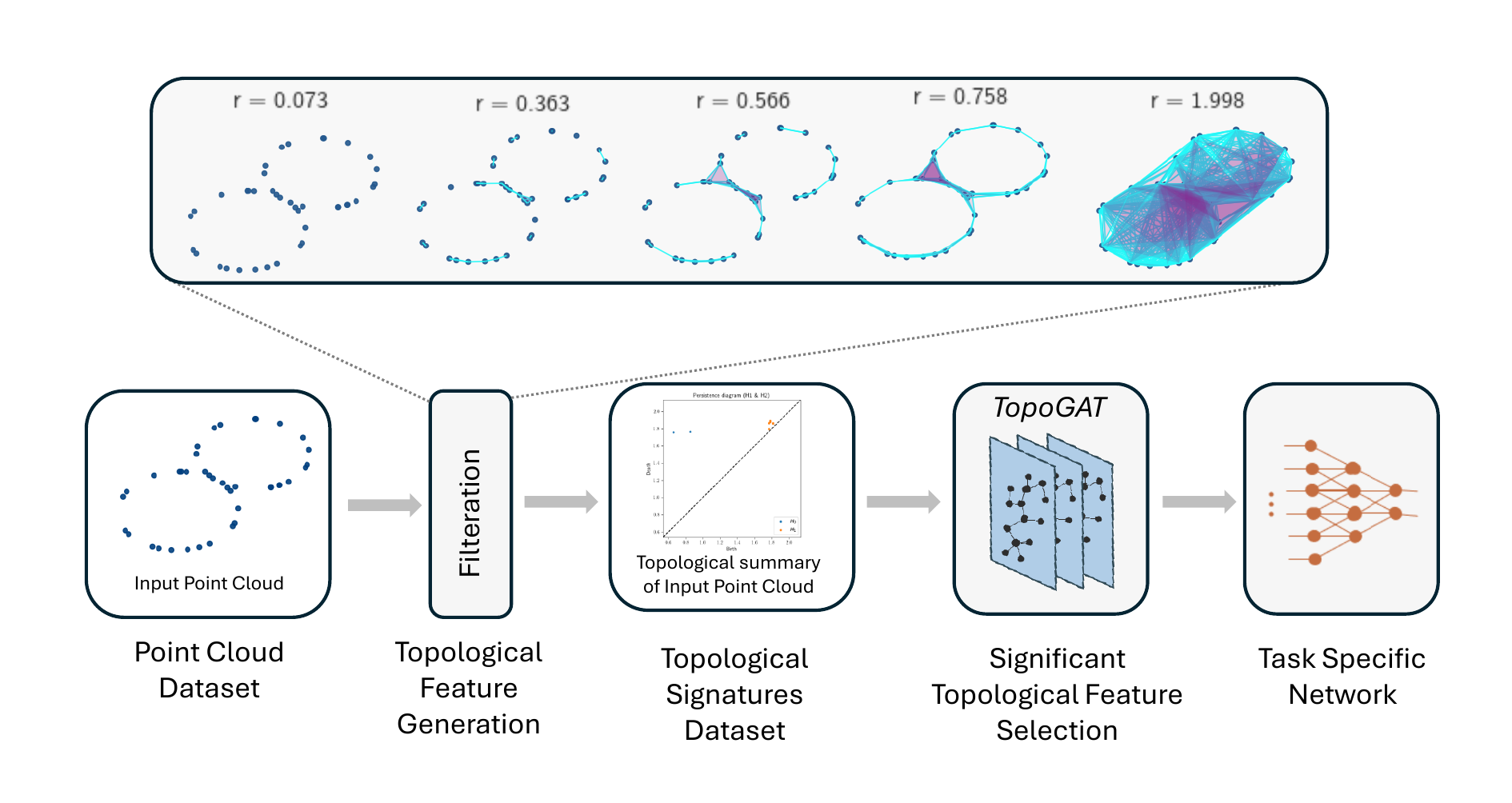}
\caption{General framework of Persistent Homology-based point cloud learning based on the proposed work. Main contribution of this work lies in the curation of the Topological-Signatures dataset and Significant Topological Feature selection Network \textit{TopoGAT}.}
\label{intro}
\end{figure}
 In existing persistent homology-based point cloud learning works \cite{hu2024topology}, a fixed number of features is selected from the pool of features generated by persistent homology. However, this approach fails to adapt to the varying complexity and scale of different datasets. A fixed feature selection strategy may overlook important topological information or fail to represent the full range of persistent features, particularly in datasets with complex structures and diverse geometric variations. To address this limitation, selecting features based on data-driven criteria rather than relying on a predetermined number ensures the retention of the most informative topological attributes for downstream tasks like classification or segmentation. In this work, we introduce a neural network architecture designed for significant topological feature selection, enabling a more adaptive and expressive representation of the underlying topological structures that enhances both model performance and generalization. Figure \ref{intro} illustrates the general framework of this work, showing how our topological-signatures dataset and the proposed significant topological feature selection network, \textit{TopoGAT}, integrate into the broader persistent-homology based point cloud learning pipeline. 

 We present a comparative study on different types of Graph Neural Networks (GNNs) for point cloud learning that leverages geometric as well as topological features for classification and part segmentation. The key contribution of this work is as follows. 
\begin{itemize}
\item \textbf{Topological Signatures Dataset:} We construct a topological signatures dataset that contains topological descriptors of ModelNet40 \cite{wu20153d} and ShapeNet~\cite{shapenet2015} that can be directly utilized to study the topology of 3D shapes or can be used to generate vector representations for further Topological Data Analysis (TDA) and 3D shape understanding applications.

\item \textbf{\textit{TopoGAT}:} We introduce a neural network that learns the topological descriptors and removes insignificant topological features by incorporating a loss function consisting of classification loss along with a novel topological loss term.

\item \textbf{Comparative study on GNNs with topological features:} We conduct a focused comparative study of various graph neural network architectures for point cloud analysis that learns geometrical, along with the topological features of the point cloud. 
\end{itemize}

The rest of the paper is organized as follows. Section 2 discusses existing related work on persistent homology-based point cloud learning and topological feature selection. Section 3 discusses the preliminaries on persistent homology. The significant topological feature selection network, along with the topological signatures dataset, is described in Section 4. Experimental setup and results, followed by an ablation study, are presented in Section 5. Finally, a conclusion and future work are discussed in Section 6.

\section{Related Work}
\label{sec2}
In this section, we will present an overview of Persistent Homology-based Machine Learning (PHML) methods for point cloud classification and segmentation, along with topological feature filtering techniques explored in prior studies.

\paragraph{PHML-based point cloud learning}
Persistent homology provides a unique way to quantify topological features of a 3D shape. As described in Section \ref{PH}, persistent homology serves as a tool to observe the change in topological features across varying filtration values. In addition, persistence of these features, measured by persistent homology, is differentiable concerning the original point cloud \cite{edelsbrunner2002topological, clough2020topological}. These properties of persistent homology make it a perfect candidate for incorporating topological insights directly into neural network training.  Many recent works of understanding point clouds \cite{peek2023synthetic}, \cite{zhou2022learning}, \cite{liu2022toposeg}, and images \cite{montufar2020can}, 3D data generation \cite{hu2024topology}, and surface reconstruction \cite{jignasu2024stitch} have explored persistent homology in integration with deep learning techniques. 

\paragraph{Topological Feature Filtering}
The variable and huge size of the persistence diagram is a major bottleneck in its application in deep learning based downstream tasks. However, a few works have discussed solutions to mitigate this problem. One popular choice is the selection of a fixed number of high-persistence topological features in persistence diagrams. Authors in \cite{hu2024topology} and \cite{liu2022toposeg} have used 16 and 300 top $k$ persistence points in the original persistence diagrams, respectively. However, a major drawback of such a filtering scheme is that it imposes an additional parameter in the deep learning setup, which becomes very crucial for the persistence diagrams with most of the topological features close to each other. Moreover, the top $k$ persistence point-based selection does not account for the topological noise for individual persistence diagrams. For example, as shown in Figure \ref{topoDATASET}, 15 topological noise points will be selected for object \textit{tv\_stand} while 10 topological noise points will be selected for object \textit{Bookshelf} in the ModelNet40 dataset for $k = 16$. Thus selection of a single optimal $k$ value for the whole dataset is often difficult and is not suitable for datasets with a variety of objects.

In a TDA paper \cite{fasy2014confidence}, several statistical methods are derived to generate confidence sets to separate topological noise from significant topological features in a persistence diagram. Authors use simple and synthetic 2D examples, such as circle and eyeglass shapes, to present proof of concept. These methods present a strong foundation on the significance of filtering of topological noise; however, these methods pose computational challenges for large datasets. In this work, we will discuss these methods in a comparative study to present the pros and cons of statistical methods along with our proposed method. 

\begin{figure}[t]
\centering
\includegraphics[width=8.0 cm]{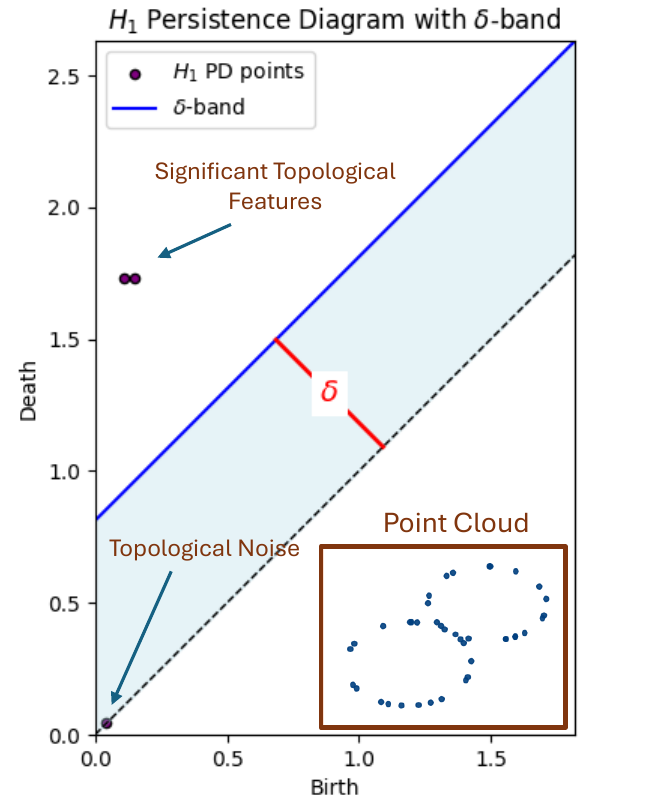}
\caption{The persistence diagram records the birth and death of topological features (clusters, loops, cavities) as points. A $\delta$-band is drawn around the diagonal: features with short lifetimes, represented by points inside this band (highlighted), are classified as topological noise. Features whose points lie outside the band are highlighted as significant persistent homology features, capturing robust and meaningful structure in the underlying shape (lower-right corner).}
\label{sph}
\end{figure} 

\section{Preliminaries}
\label{prelim}

In this section, we define persistent homology and other necessary fundamentals to provide insights into TDA-based point cloud learning. We refer the reader to \cite{dey2022computational} for more advanced details.

\subsection{Persistent Homology}\label{PH}
Persistent homology is a powerful tool to quantify the topology of a 3D shape. Persistent homology identifies and tracks topological features such as connected components, loops, and voids across multiple spatial scales. Therefore, it captures invariant structural properties of a 3D shape that remain stable under continuous deformations.

Persistent homology has been used to extract global and multi-scale descriptors by constructing filtrations over point clouds that result in topological summaries such as persistence diagrams or persistence barcodes. These topological summaries complement geometric features in order to improve generalization and robustness in downstream tasks such as object recognition or similarity comparison. The key idea of persistent homology is to follow the appearance and disappearance of topological features over the filtration. 

Let X be a finite metric space in $X \in \mathbb{R}^{d}$ (such as a 3D point cloud). A filtration $\{K_\epsilon\}_{\epsilon\geq0}$ is a nested sequence of simplicial complexes $K_{\epsilon}$ that includes simplices formed by connecting points that lie within a distance $\epsilon$. In each homological dimension $k$, the collection of birth-death pairs $(b, d)$ of $k$-dimensional homology classes forms a persistence diagram $Dgm_{k}(X)$. Thus, a persistence diagram is a multiset of points $(b, d) \in \mathbb{R}^{2} \cup \{ \infty \}$ with $b \leq d$ that corresponds to topological features that are born at filtration value $b$ and die at $d$. The difference $d - b$ is known as the persistence of the topological feature and represents its lifetime across the filtration. 

The persistence $d - b$ of the topological feature measures how long a topological feature remains present as the filtration evolves. Features with large persistence represent stable and meaningful structures in the data. On the other hand, features with small persistence are often considered as topological noise. Persistent homology builds on this idea by organizing these birth-death pairs across all scales and producing a complete multi-scale summary of the topological structure of the point cloud. In this way, the collection of persistence values encodes how the shape of the data changes with the filtration parameter, making persistent homology a powerful framework for distinguishing significant geometric features from noise.
\begin{figure*}[t]
\centering
\includegraphics[width=16 cm]{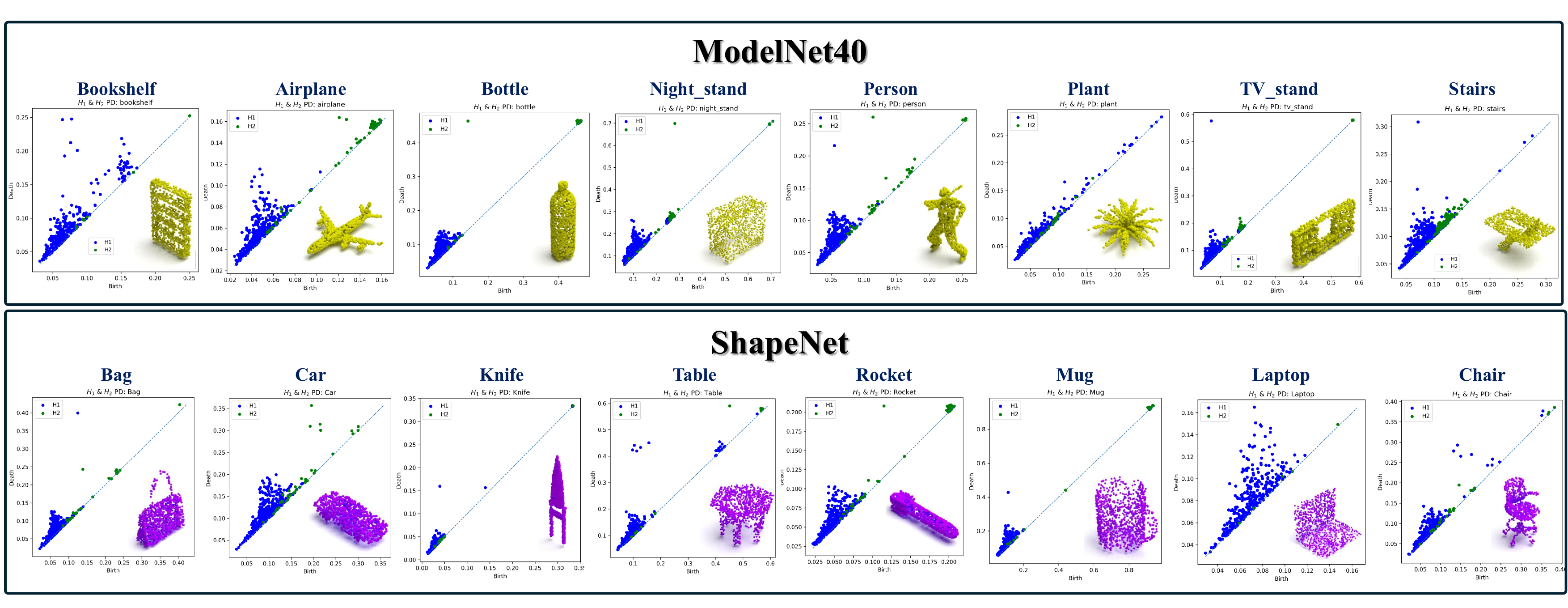}
\caption{\textbf{Topological Signatures Dataset:} Example $H_{1}$ and $H_{2}$ persistence diagrams of eight objects each from different classes of ModelNet40 \cite{wu20153d} (in green) and ShapeNet \cite{shapenet2015} (in purple) respectively. Generating these PDs took over 860 hours on an Intel Silver 4216 Cascade Lake CPU (187 GB RAM) with NVIDIA V100 Volta GPU (32G HBM2 memory) across ModelNet40 and ShapeNet, reflecting the high computational cost of persistent homology at scale, and underscoring the utility of our precomputed topological signatures. In each diagram, blue points denote $H_{1}$ PD points and green points denote $H_{2}$ PD points. } 
\label{topoDATASET}
\end{figure*}
\subsection{Significant Persistent Homology Features}
Persistent homology provides rich topological summaries of 3D shapes, but not every feature in the persistence diagram carries geometric meaning for the underlying shape. A crucial question is how to distinguish genuine shape characteristics from topological noise. In point cloud data, insignificant features with low persistence commonly arise due to noise, sparse sampling, or minor surface perturbations, manifesting as points close to the diagonal in the persistence diagram \cite{fasy2014confidence}.

While statistical approaches (e.g., \cite{fasy2014confidence}) offer data-driven criteria for distinguishing significant topological features in persistence diagrams, in this work we focus on a geometric notion of significance that is particularly well suited to 3D point cloud analysis. The central idea is that a feature in the persistence diagram is geometrically significant if it persists over a sufficiently large range of filtration scales, reflecting stability and prominence in the underlying shape.

Formally, let $X \subset \mathbb{R}^3$ be a finite point cloud, and let $\{K_\epsilon\}_{\epsilon\geq0}$ be a filtration built from $X$ (e.g., via the Vietoris–Rips or Čech complex). Each $k$-dimensional feature is associated with a birth-death pair $(b, d)$ in the persistence diagram $\mathrm{Dgm}_k(X)$, where $b$ and $d$ are the parameter values at which the feature appears and disappears.

We define a geometric significance threshold $\delta > 0$, and classify a feature as \textit{significant} if
\[
d - b > \delta
\]

Here, $\delta$ represents the minimal lifetime (across scale) that a feature must exhibit to be considered robust; features with $d-b \leq \delta$ are treated as geometric noise, typically arising from sampling artifacts or small fluctuations.

This geometric definition is closely tied to the stability of persistent homology~\cite{cohen2007stability}. 

\begin{theorem}\label{thm1}
   \textbf{Bottleneck Stability Theorem} (Cohen-Steiner et al., 2007~\cite{cohen2007stability}). 
\emph{Let $X, X' \subset \mathbb{R}^d$ be finite metric spaces. Then the bottleneck distance between their $k$-dimensional persistence diagrams is bounded by their Hausdorff distance:}
\[
W_\infty\bigl(\mathrm{Dgm}_k(X),\, \mathrm{Dgm}_k(X')\bigr) \leq H(X, X').
\] 
\end{theorem}
where $H(X, X')$ denotes the Hausdorff distance between $X$ and $X'$. Here, $W_\infty$ is the bottleneck distance, which measures the greatest difference that must occur between matched pairs of features (including possibly matching features to the diagonal) in the two persistence diagrams, under the best possible matching. This fundamental result ensures that persistent homology features are robust under small perturbations of the underlying space. We formalize the robustness of \textit{significant} persistent homology features in Lemma \ref{lm1}.

\begin{lemma}\label{lm1}
\emph{Let $X \subset \mathbb{R}^3$ be a finite point cloud, and suppose $(b, d) \in \mathrm{Dgm}_k(X)$ has persistence $d-b > \delta$ for some $\delta > 0$. Then for any perturbation $X'$ with Hausdorff distance $H(X, X') < \delta/2$, there exists a corresponding feature in $\mathrm{Dgm}_k(X')$ with persistence at least $d-b-\delta$.} 
\end{lemma} 

\begin{proof}

Let $(b, d) \in \mathrm{Dgm}_k(X)$ be a feature with $d-b > \delta$. By stability theorem (Theorem \ref{thm1}), for each point $(b, d)$, there exists a point $(b', d')$ in $\mathrm{Dgm}_k(X')$ such that
\begin{equation}\label{eq1}
|b - b'| \leq H(X, X').
\end{equation}

\begin{equation}\label{eq2}
|d - d'| \leq H(X, X').
\end{equation}

From Equations \ref{eq1}-\ref{eq2}, it follows that the persistence in $X'$ satisfies
\[
d' - b' \geq d - b - 2 H(X, X').
\]
If $H(X, X') < \delta/2$, then
\[
d' - b' > d - b - \delta >0.
\]
Therefore, the feature survives as long as the perturbation is less than half the persistence. 
Any feature with persistence more than $\delta$ in the original data is guaranteed to have a corresponding feature in any sufficiently close point set $(H(X, X') < \delta/2)$, with persistence above $0$.
\end{proof}
Lemma \ref{lm1} ensures that \textit{significant} features reflect genuine, large-scale geometry of the underlying shape, rather than ephemeral noise. In the context of 3D point cloud shape analysis, the geometric persistence threshold $\delta$ provides a principled means to automatically identify robust shape properties, such as prominent tunnels or cavities, while systematically filtering out spurious topological artifacts caused by noise or sampling irregularities. Practically, this approach is visualized by overlaying a band of width $\delta$ along the diagonal of the persistence diagram; only features whose persistence exceeds this band are deemed geometrically significant and interpreted as reliable topological signatures, as illustrated in Fig.~\ref{sph}.

While this framework provides a robust theoretical foundation for geometric significance in persistent homology, it is important to acknowledge that practical machine learning models such as deep neural networks may not exactly recover all theoretically significant features. Instead, such models strive to capture representations that approximate the set of robust, geometrically meaningful features. This interplay highlights both the utility and limitations of persistent homology in learning; although the geometric definitions and associated stability theorems establish what is theoretically significant, deep neural networks leverage persistent homology to encode informative and generalizable shape signatures for downstream tasks, even if the correspondence with the ideally significant features is only approximate. 
\begin{figure*}[t]
\centering
\includegraphics[width=16 cm]{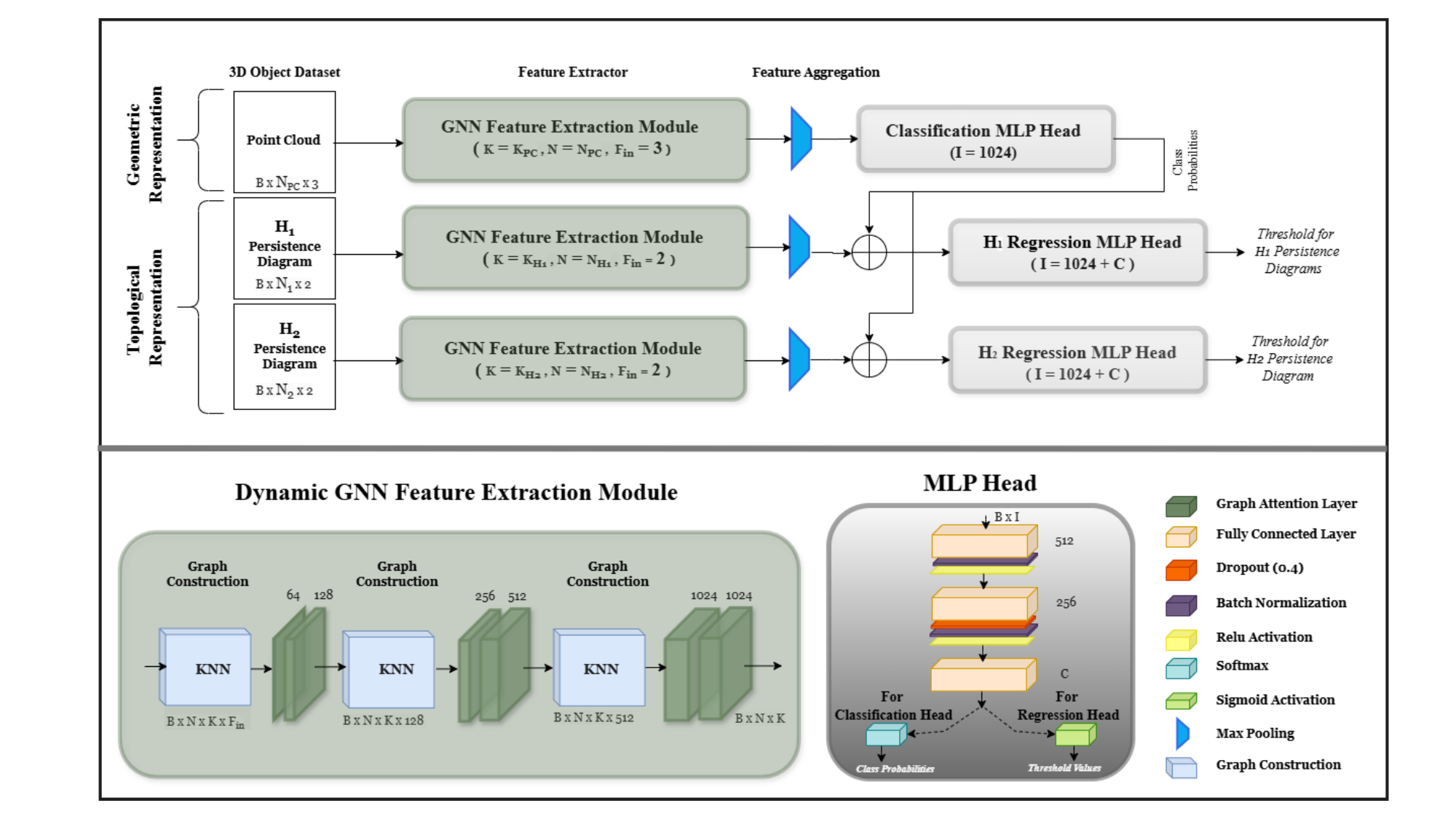}
\caption{\textbf{\textit{TopoGAT} Architecture:} Three identical GNN feature extraction modules individually construct dynamic KNN graphs of a point cloud along with its topological representations while utilizing GAT layers to extract geometric as well as topological features. Classification MLP head predicts logits for each point cloud, which is further utilized to predict a threshold value for each point cloud using Regression MLP heads for each homological dimension. Classification and regression MLP heads differ in the input size and output layer.}
\label{architecture}
\end{figure*}
\section{Network Architecture}\label{NWarch}
 In this section, we will discuss the details of our topological signatures dataset and proposed significant topological feature selection network \textit{TopoGAT}, along with the persistence diagram filtering mechanism and the proposed topological loss function that guides the \textit{TopoGAT} for efficient significant feature selection.

\subsection{Topological Signatures Datasets}

 We introduce the topological signatures datasets (\textit{ModelNet40-PD} and \textit{ShapeNet-PD}), topological counterparts of the ModelNet40~\cite{wu20153d} and ShapeNet~\cite{shapenet2015} datasets. These datasets contain persistence diagrams of each point cloud in ModelNet40 \cite{wu20153d} and ShapeNet~\cite{shapenet2015} that can be further used for topological analysis of 3D shapes. These topological datasets contain homology-1 $H_{1}$ and homology-2 $H_{2}$ persistence diagrams of each point cloud. We use Vietoris–Rips simplicial complex to generate persistence diagrams from point cloud data to further calculate persistent homology by identifying the topological features that persist across different scales in a dataset. Initially, we sampled each point cloud to 1024 points using random sampling. Next, the filtration of each point cloud is performed to identify topological features such as connected components, loops, and voids. In this dataset, we provide $H_{1}$ and $H_{2}$ PDs and do not include $H_{0}$ PDs due to their less significance in the 3D shape learning application. 
 
 To ensure that all persistence diagrams are of equal size to be compatible with a deep learning setup, we padded PDs based on the following criteria. We first determine the size of the largest $H_{1}$ PD across the whole $H_{1}$ topological dataset and then expand each PD to match this size. For this, we pad it with the copies of the smallest persistence point of each $H_{1}$ PD, repeated as many times as needed to compensate for the size deficiency. Similarly, $H_{2}$ PDs are padded with their smallest persistence points. Replicating the least persistent points in each PD minimizes the distortion introduced by padding due to its proximity to the diagonal. This ensures minimal impact on topological distances such as Wasserstein distance or Bottleneck distances, which play a major role in topological learning of these datasets.

Figure \ref{topoDATASET} shows a set of $H_{1}$ and $H_{2}$ persistence diagrams of eight objects, each selected from different classes of ModelNet40 \cite{wu20153d} (in green) and ShapeNet \cite{shapenet2015} (in purple), respectively. In this figure, we selected point clouds from different classes to show the diversity in the number of high persistence (significant) $H_{1}$ and $H_{2}$ points in our dataset. For example, our dataset consists of single high-persistence $H_{2}$ point objects such as \textit{bottle}, \textit{night\_stand}, \textit{person} from ModelNet40, and \textit{table} and \textit{rocket} from ShapeNet. Similarly, \textit{person}, \textit{tv\_stand}, and \textit{stairs} from ModelNet40 and \textit{bag} from ShapeNet have one high persistence $H_{1}$ point, which can be visually verified by observing corresponding point clouds. Additionally, our dataset has objects that have multiple high persistence $H_{1}$ points, such as \textit{bookshelf} in ModelNet40 and \textit{table} in ShapeNet with 5 points. Our dataset also consists of complex objects like \textit{plant} in ModelNet40 and \textit{laptop} in ShapeNet, which are difficult to interpret because their persistence diagrams contain many low-persistence topological features, making it challenging to clearly identify and separate the significant points.

The combined time duration for generating these PDs across the ModelNet40 (approximately 360 hours) and ShapeNet (approximately 507 hours) datasets is over 860 hours, which demonstrates the extensive computational time involved in applying persistent homology to large-scale 3D deep learning tasks. Therefore, these pre-computed topological datasets provide the research community with immediate access to advanced topological features for rapid exploration and development of topology-aware models for point cloud processing.

\subsection{Topological Feature Selection Network: \textit{TopoGAT}}
To select significant features from the original $H_{1}$ PD and $H_{2}$ PD, we present a GNN-based neural network called \textit{TopoGAT}. As shown in Figure \ref{architecture}, \textit{TopoGAT} is a three-branch network to learn geometric features from the point cloud along with topological features from $H_{1}$-PDs and $H_{2}$-PDs. The features extracted directly from point clouds are used in a classification MLP head to train model parameters for given classification labels. The topological features extracted from PD, along with the class probabilities derived from the classification head, are concatenated to serve as input to the regression MLP heads that predict threshold values for each PD.

Feature extraction from each representation is done individually using three similar architecture modules named \textit{GNN feature extraction module}, which consists of three stacked dual GAT layer blocks, with KNN dynamically recomputed after each block. Dynamic graph learning adaptively updates neighborhood connections during training, enabling the model to capture evolving geometric and topological relationships. On the other hand, graph attention allows the model to focus on the most relevant geometric or topological relationships. 

Three similar structure MLPs are used for the classification head and both the regression heads, except for the output layer, which is softmax in the classification head and sigmoid activation in the regression heads. As a common practice adopted in most of the point cloud classification networks, softmax converts logits into a probability distribution over mutually exclusive classes by increasing the probability of the object class while decreasing the probability values for the other categories. In contrast, sigmoid activation in the regression head ensures independent continuous values bounded in $[0,1]$. Moreover, in persistence diagrams, the topological features may differ across classes due to variation in the topology of inter-class objects. Therefore, the input to each regression head is the concatenation of the corresponding topological feature map and the class probabilities predicted by the classifier. Thus, regression is conditioned on the predicted class, resulting in class-specific threshold values. This conditioning allows the regression head to personalize thresholds for each topology pattern instead of learning a single average threshold; thus, it improves the relevance and precision of the regression head. Figure \ref{TopoGAT_filter} shows a few examples of original PDs and their \textit{TopoGAT}-selected topological features from different classes in the ModelNet40 and ShapeNet datasets. 
\begin{figure*}[h]
\centering
\includegraphics[width=16 cm]{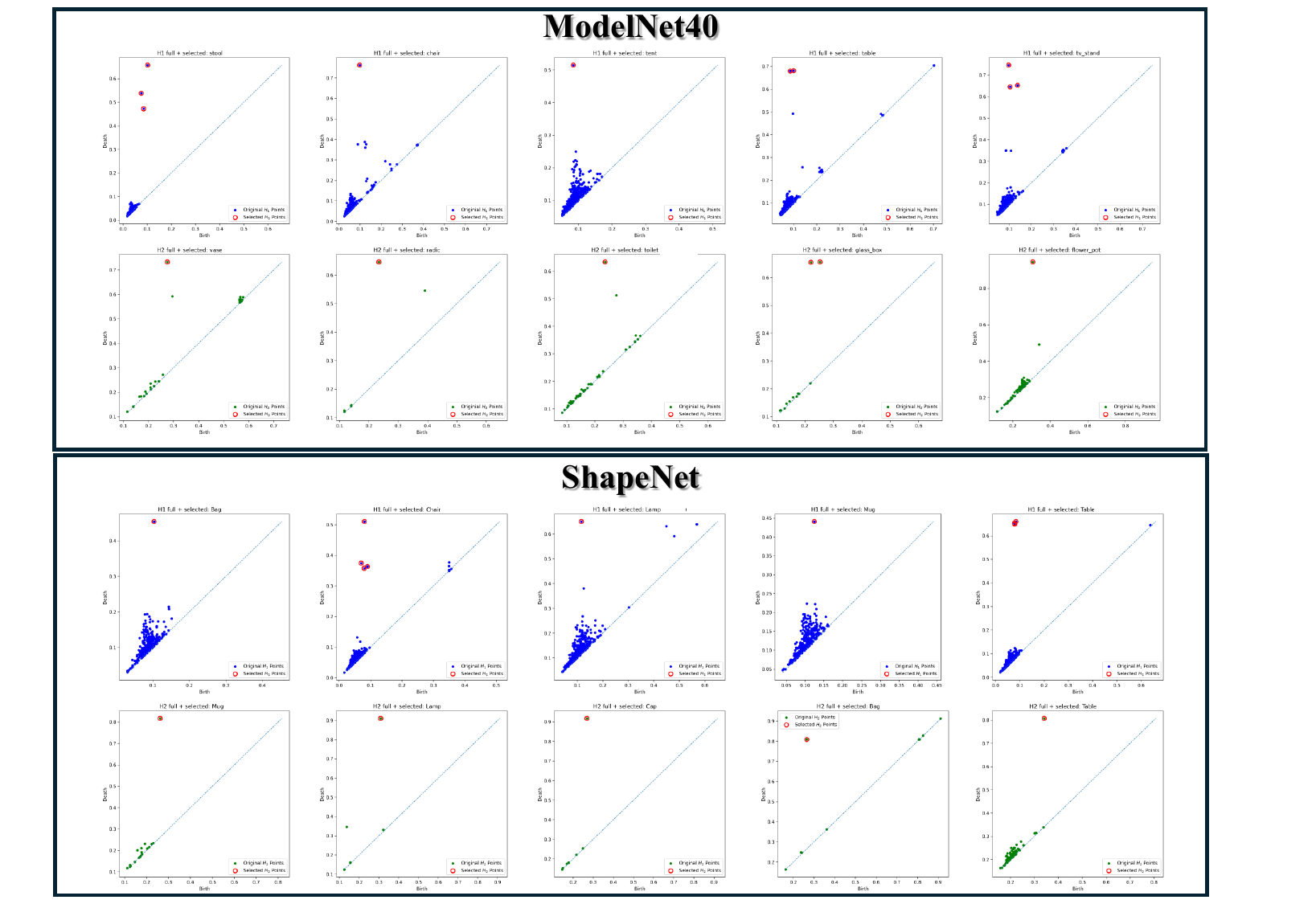}
\caption{Example of original PD (blue indicate $H_{1}$ features and green indicate $H_{2}$ features) with significant features selected by \textit{TopoGAT} (red circle).}
\label{TopoGAT_filter}
\end{figure*}
\subsection{Persistence Diagram Filtering}
The threshold values obtained from the \textit{TopoGAT} regression head are further used to filter out the insignificant points in the persistence diagrams. These filtered persistence diagrams are utilized in the topological loss function mentioned in Section \ref{TopolossDef}. A simpler way to filter a persistence diagram is by eliminating the topological features with persistence higher than the threshold value $\lambda$. However, such hard selection of top-$k$ persistent points is non-differentiable, resulting in no gradient flow from the loss back to the threshold prediction head. Therefore, to ensure end-to-end differentiability of this filtering process, we compute a sigmoid-based soft mask to filter out low-persistence points in the persistence diagram, as shown in Equation \ref{softFilteringMask}. 
    \begin{equation}
    \text{mask}_{i} = \sigma(\eta * (\text{p}_{i} - \lambda)) 
    \label{softFilteringMask}
    \end{equation}
This filtering technique allows the network to flow a gradient through the sigmoid $\sigma(.)$ and hence learn better thresholds. Additionally, sigmoid helps in avoiding abrupt drops in gradient flow. The hyperparameter $\eta$ controls the softness of the mask, and a large $\eta$ may cause vanishing gradients for threshold values, while a smaller $\eta$ will result in softer transitions. Hence, it is crucial to select a proper value of $\eta$. Therefore, we use a learnable $\eta$ that will facilitate the network to adaptively decide how strict the filtering should be, depending on the PD, class, as well as the task.

The mask corresponding to each PD is element-wise multiplied with its original PD $\text{PD}_{\text{org}}$ to yield the reduced PD $\text{PD}_{\text{s}}$, defined as follows.

    \begin{equation}
    \text{PD}_{\text{s}} = \text{mask} * \text{PD}_{\text{org}} 
    \label{softFiltering}
    \end{equation}
\subsection{Topological Loss} \label{TopolossDef}
One of the methods to leverage topology in deep learning is by incorporating a topological penalty in the overall loss function \cite{zia2024topological}. Such a strategy (loss-based integration) eliminates the need for additional topological layers, and hence it is straightforward considering the differentiability of the topological loss term. Therefore, to guide the training of \textit{TopoGAT} for significant topological point selection, we propose a composite loss function that consists of a geometric loss term $\mathcal{L}_{Geo}$ and a novel topological loss term $\mathcal{L}_{Topo}$ designed to capture persistent structural features of the input point cloud. 
    \begin{equation}
    \mathcal{L} = \mathcal{L}_{Geo} + \mathcal{L}_{Topo}
    \label{LossFunction}
    \end{equation}

We define a topological loss term $\mathcal{L}_{\text{Topo}}$ to learn topological features extracted from the persistence diagrams $H_{1}$ and $H_{2}$ of the point cloud. We use this topological loss term to select significant persistence features while penalizing the removal of topologically informative features that preserve the underlying shape of the data. Therefore, this topological loss term comprises three components: the Wasserstein distance between the original and reduced persistence diagrams, the persistent entropy of the reduced diagram, and the rate of reduction in the persistence diagram size. 

\paragraph{\textbf{Wasserstein Distance Loss (WL)}} The Wasserstein distance between two persistence diagrams measures their similarity by considering the cumulative cost of optimally transporting features from one persistence diagram to the other. It aggregates the overall differences across all matched features. Therefore, it allows small discrepancies spread across many points to be captured, rather than being dominated by a single worst-case mismatch. Within the loss function, this topological distance encourages the reduced persistence diagram to remain close to the original by minimizing the total transport cost of topological features. This ensures that both prominent and moderately significant features are preserved in a balanced way, discouraging excessive distortion or loss of structural information in the reduction process.

\paragraph{\textbf{Persistent Entropy Loss (EL)}} The persistent entropy is a measure of the complexity and information content of topological features in a persistence diagram. The persistent entropy of the reduced PD, as a loss term, discourages over-simplification of the PD, resulting in the selection of diverse and informative features, rather than oversimplifying the diagram by keeping only a few very dominant ones. By penalizing entropy loss, we avoid degenerate cases where only a minimal number of features are preserved, which could compromise downstream performance, especially in complex data like ModelNet40 point clouds. 

\paragraph{\textbf{Reduction Loss (RL)}} The third Component of our topological loss is the reduction loss, which is the fraction of points removed from the original persistence diagram. It is the ratio of the size of the reduced diagram to that of the original diagram. This term serves as a soft penalty for retaining too many topological features, thereby providing a balance between informativeness and compactness. 

By comprising these abovementioned loss components, we define our topological loss term as follows.

    \begin{equation}
        \begin{aligned}
            \mathcal{L}_{\text{Topo}} &= \alpha \cdot \mathcal{L}_{\text{WD}} + \beta \cdot \mathcal{L}_{\text{PE}} + \gamma \cdot \mathcal{L}_{\text{R}}, \\
            \text{where} \quad &\alpha + \beta + \gamma = 1, \quad \alpha, \beta, \gamma \in (0,1)
        \end{aligned}
    \label{eq:TopoLoss}
    \end{equation}
Instead of fixed weights for different components of $\mathcal{L}_{\text{Topo}}$, we use learnable $\alpha$, $\beta$, and $\gamma$ that allow the network to adaptively balance the scale and importance of each term throughout the training. To ensure stability and differentiability, a softmax operation on these learnable topological weights is applied.

In the geometric loss term, we use the Softmax Cross-Entropy Loss function to perform multi-class classification using geometric features extracted from the point cloud. The geometric loss term utilizes the cross-entropy between predicted class probabilities and ground truth labels to supervise the classification branch, which in turn guides the regression block responsible for predicting threshold values for identifying significant features in the persistence diagram.
\begin{table*}[t!]
\centering
\renewcommand{\arraystretch}{1.2}
\setlength{\tabcolsep}{4pt}

\captionsetup{singlelinecheck=false, justification=centering}
\caption{Comparison of Significant Feature Selection Methods with \textit{TopoGAT} for ModelNet40 \cite{wu20153d} dataset.}
\begin{tabular}{|p{2.0cm}|p{1.2cm}|p{1.2cm}|p{1.2cm}|p{1.2cm}|p{1.2cm}|p{1.2cm}|p{2.0cm}|}
\hline
\textbf{Feature Selection Method} &
\multicolumn{2}{c|}{\textbf{Mean WD$\downarrow$}} &
\multicolumn{2}{c|}{\textbf{Mean BD}$\downarrow$} &
\multicolumn{2}{c|}{\textbf{Mean PE Difference}$\downarrow$} &
\textbf{Processing Time (s)} \\
\cline{2-7}
& \textbf{$H_{1}$-PD} & \textbf{$H_{2}$-PD}
& \textbf{$H_{1}$-PD} & \textbf{$H_{2}$-PD}
& \textbf{$H_{1}$-PD} & \textbf{$H_{2}$-PD} & \\
\hline

Method 1 \cite{fasy2014confidence} & 0.056 & 0.069 & 0.239 & 0.233 & 52.7 & 3.43 & 1182720.0 \\
Method 2 \cite{fasy2014confidence} & 0.028 & 0.040  & 0.060 & 0.039 & 23.5 & 1.38 & \textbf{778.7} \\
Method 3 \cite{fasy2014confidence} & 0.037 & 0.054 & \textbf{0.056} & \textbf{0.032} & 23.5 & 1.40   & 162083.9 \\
\textit{TopoGAT} & \textbf{0.024} & \textbf{0.026} & 0.061 & 0.045 & \textbf{23.4} & \textbf{1.37} & 149099.2 \\
\hline
\end{tabular}\label{tab1}
\end{table*}

\begin{table*}[t!]
\centering
\renewcommand{\arraystretch}{1.2}
\setlength{\tabcolsep}{4pt}

\captionsetup{singlelinecheck=false, justification=centering}
\caption{Comparison of Significant Feature Selection Methods with \textit{TopoGAT} for ShapeNet \cite{shapenet2015} dataset.}
\begin{tabular}{|p{2.0cm}|p{1.2cm}|p{1.2cm}|p{1.2cm}|p{1.2cm}|p{1.2cm}|p{1.2cm}|p{2.0cm}|}
\hline
\textbf{Feature Selection Method} &
\multicolumn{2}{c|}{\textbf{Mean WD$\downarrow$}} &
\multicolumn{2}{c|}{\textbf{Mean BD$\downarrow$}} &
\multicolumn{2}{c|}{\textbf{Mean PE Difference$\downarrow$}} &
\textbf{Processing Time (s)} \\
\cline{2-7}
& \textbf{$H_{1}$-PD} & \textbf{$H_{2}$-PD}
& \textbf{$H_{1}$-PD} & \textbf{$H_{2}$-PD}
& \textbf{$H_{1}$-PD} & \textbf{$H_{2}$-PD} & \\
\hline

Method 1 \cite{fasy2014confidence} & 0.049 & 0.029 & 0.268 & 0.106 & 38.52 & 2.57 & 1439262.2  \\
Method 2 \cite{fasy2014confidence} & 0.033 & 0.005 & 0.059 & 0.021 & 15.49 & 1.09 & \textbf{1117.2} \\
Method 3 \cite{fasy2014confidence} & 0.048 & 0.007 & \textbf{0.053} & \textbf{0.019} & 15.48 & 1.09  & 210129.8 \\
\textit{TopoGAT} & \textbf{0.029} & \textbf{0.004} & 0.061 & 0.021 & \textbf{15.46} & \textbf{1.08} & 183953.5 \\
\hline
\end{tabular}\label{tab2}
\end{table*}
  
\section{Evaluation}

In this section, we discuss the comparative study of existing statistical topological feature selection techniques and proposed \textit{TopoGAT}, along with the performance of \textit{TopoGAT}-based feature selection in downstream tasks classification and part-segmentation. For a comparative study of existing statistical topological feature selection techniques, we use three feature selection methods from \cite{fasy2014confidence}: Method 1, Method 2, and Method 3, which follow statistical approaches to estimate confidence bands for significant topological features.

\subsection{Computational Environment}
The proposed architecture is developed in Python 3.11 and utilizes PyTorch 2.0 for artificial neural network layers. We use ModelNet40 \cite{wu20153d} and Shapenet \cite{shapenet2015} datasets to generate their corresponding topological equivalent datasets. To generate homology-1 $H_{1}$ and homology-2 $H_{2}$ persistence diagrams of each point cloud in both datasets, we use the high-performance persistent homology package Ripser \cite{ctralie2018ripser}. Ripser computes Vietoris–Rips persistence diagrams from point cloud data to generate topological features that persist across different scales in a dataset. Please note that due to excessive recursion in the algorithm, we reduced the size of one point cloud in ModelNet40 and 10 point clouds in ShapeNet from 1024 to 920 to deal with the memory management issue. Mathematical computations and geometry processing are done using Sci-kit Learn and NumPy libraries. For significant topological feature selection and training and testing for classification experiments, we use an Intel Silver 4216 Cascade Lake CPU with NVIDIA V100 Volta GPU. For part-segmentation, we use an AMD EPYC 9454 (Zen 4) CPU with an NVIDIA H100 SXM5 GPU. 
 
\subsection{Performance Metrics}
 For the quantitative analysis of comparisons among different significant PD point selection methods, we use the mean- and classwise Wasserstein distance, the mean- and classwise bottleneck distance, and the mean- and classwise persistent entropy difference between the original and reduced persistence diagrams, along with runtime. 

The bottleneck distance (BD) \cite{fugacci2016persistent, agami2023comparison} between two persistence diagrams measures the similarity between the persistence diagrams and quantifies the closeness of reduced PD to its original PD. This ensures the minimal cost of matching features between them while allowing for some points of bigger PD, i.e., original PD, to be matched to the diagonal. Thus, the bottleneck distance quantifies the maximum deviation between matched topological features of the original and reduced diagrams. Therefore, a low bottleneck distance between the original and reduced diagrams ensures that the reduced persistence diagram preserves the most prominent topological features of the original by minimizing the worst-case loss of information. 

In contrast to the bottleneck distance, which focuses on the single largest deviation between matched features, the Wasserstein distance (WD) \cite{berwald2018computing, agami2023comparison} accounts for all differences between corresponding topological features in the original PD and its reduced PD. WD is sensitive to significant as well as short-lived features. Therefore, a low WD value indicates high similarity in the overall distribution of topological features in the original PD and reduced PD. 

As mentioned in section \ref{TopolossDef}, persistent entropy \cite{merelli2015topological} strikes a balance between simplification and topological representativity. Thus, the persistence entropy difference between two persistence diagrams measures the absolute difference between their entropy values by transforming each diagram into a probability distribution based on the normalized lifetimes of its topological features. Therefore, a low difference in persistence entropy of two PDs shows the high similarity in their distributions of feature lifetimes and hence a comparable topological complexity.

For the downstream task performance of the proposed deep learning based significant PD-point selection network, we use overall classification accuracy (OA) and mean accuracy (mAcc). For part segmentation experiments, we use Intersection-over-Union (IoU), point-wise accuracy (pAcc), and class-wise IoU.

\begin{figure*}[h]
\centering
\includegraphics[width=16 cm]{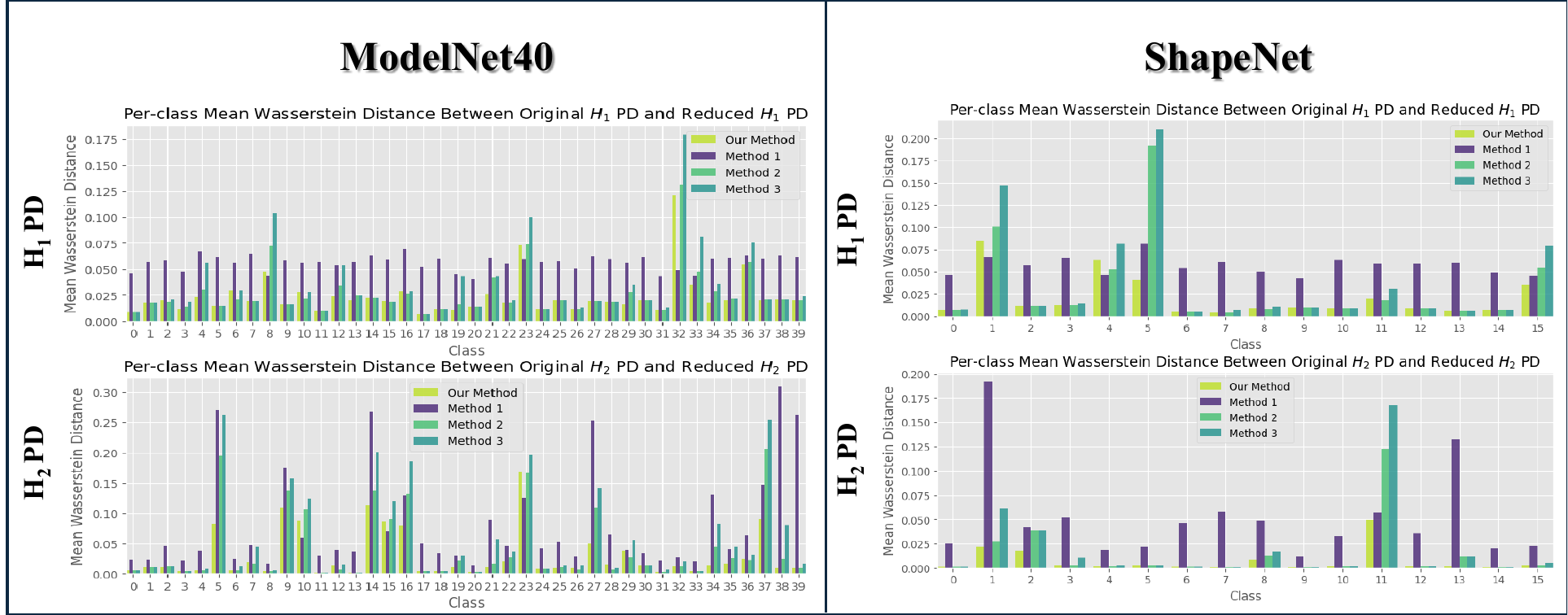}
\caption{Class-wise mean Wasserstein Distance between original PD and reduced PD for ModelNet40 (left) and ShapeNet (right).}
\label{WD}
\end{figure*}

\begin{figure*}[h]
\centering
\includegraphics[width=16 cm]{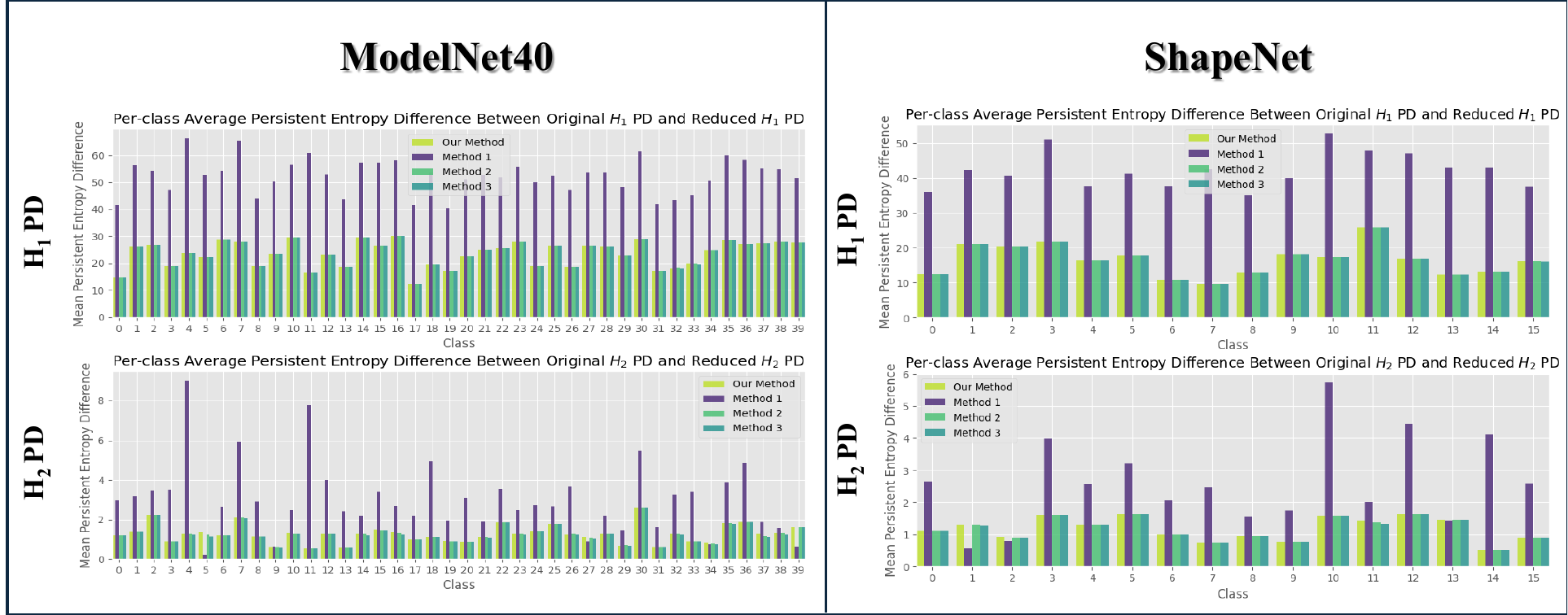}
\caption{Class-wise mean difference between Persistent Entropies of original PD and reduced PD for ModelNet40 (left) and ShapeNet (right).}
\label{PE}
\end{figure*}

\begin{figure*}[h!t]
\centering
\includegraphics[width=16 cm]{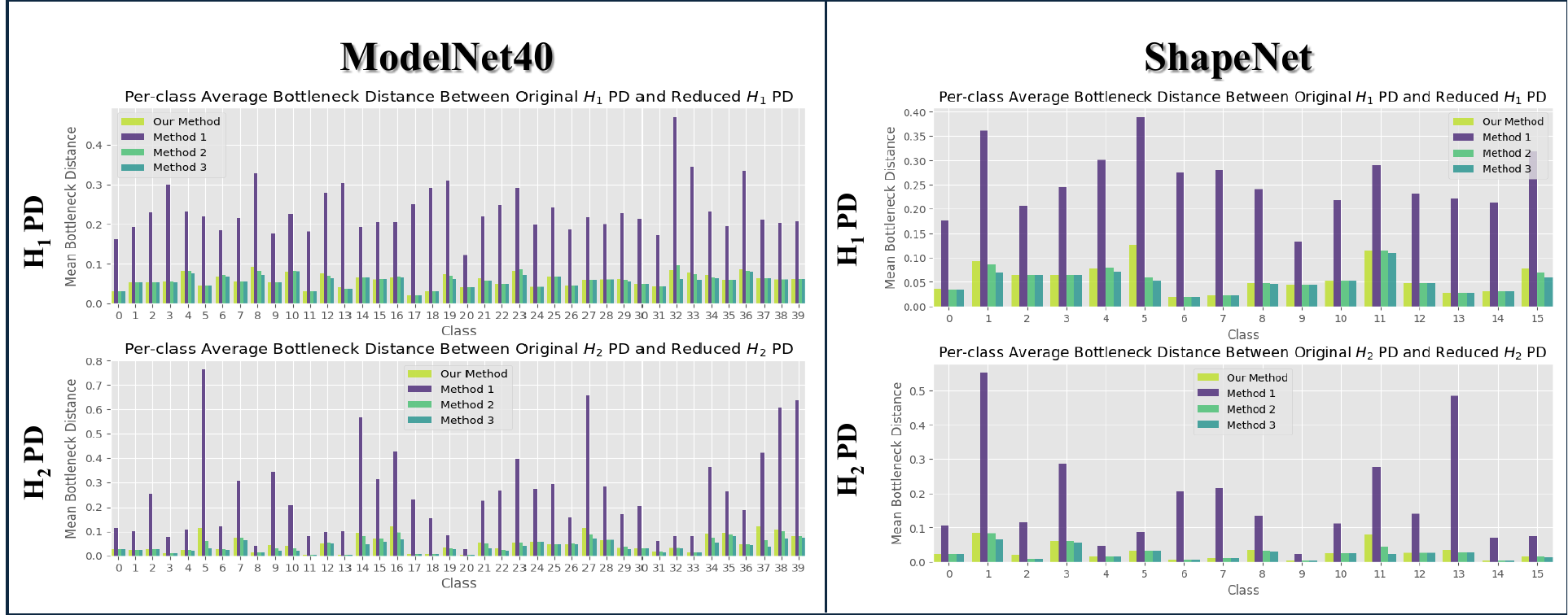}
\caption{Class-wise mean Bottleneck Distance between original PD and reduced PD for ModelNet40 (left) and ShapeNet (right).}
\label{BD}
\end{figure*}

\subsection{Significant PD Points Selection: Statistical Methods Vs. \textit{TopoGAT}}
 In this section, we compare the classical methods described in \cite{fasy2014confidence} with the proposed deep learning network for the selection of significant topological features. We implemented these methods based on the mathematical description provided in the paper.

 \textit{The concentration of measure} (Method 2) and \textit{the method of shells} (Method 3) from \cite{fasy2014confidence} present a threshold value for topological features in the persistence diagram, which is based on the point cloud. These threshold values are directly applied to our persistence diagram dataset to filter the significant topological features. Moreover, these methods derive one threshold value from each PC to filter insignificant topological features from both $H_{1}$ and $H_{2}$ persistence diagrams.
 
 The \textit{subsampling} (Method 1)~\cite{fasy2014confidence} repeatedly draws small subsets of the data and computes a statistic on each subset. Further, the empirical distribution of these statistics is used to define a conservative confidence bound for the parameter of interest. Unlike method 1, which uses Rips complex-based filtration, method 4 utilizes the cubical complex for filtration to generate persistence diagrams with significant topological features, and therefore, we do not include method 4 in this study due to the simplicial complex difference. 

 Tables \ref{tab1} and \ref{tab2} present the comparative study of these methods in terms of overall mean Wasserstein distance (mean WD), mean bottleneck distance (mean BD), and mean persistent entropy difference (mean PE diff) metrics for ModelNet40 and ShapeNet datasets, respectively. Similarly, Figures \ref{WD}, \ref{PE}, and \ref{BD} show the classwise values of mean WD, mean BD, and mean PE diff for both the datasets. The tables and figures show that Method 1 involves a longer selection process and yields higher overall as well as classwise average values across all three metrics. In contrast, \textit{TopoGAT} achieves the lowest mean WD and mean PE diff among all four methods, demonstrating that its reduced persistence diagrams are most similar to the original diagrams in terms of distribution. On the other hand, \textit{TopoGAT} attains the third-lowest mean BD, ranked after method 3 and method 2, with minimal differences in their scores. This ensures that, similar to methods 3 and 2, the proposed \textit{TopoGAT}-based feature selection method preserves the most prominent topological features of the original PD. Similarly, Figures \ref{WD}, \ref{PE}, and \ref{BD} show that \textit{TopoGAT} attains lower class-wise average values, comparable to methods 2 and 3 across most classes in both ModelNet40 and ShapeNet, with \textit{TopoGAT} outperforming them in a few classes. Moreover, \textit{TopoGAT} is faster than both method 1 and method 3.  

Figure \ref{pdcomp} presents a few examples of persistence diagrams of different objects from both datasets with original PD points (in blue for $H_{1}$ and in green for $H_{2}$), \textit{TopoGAT}-selected PD points (encircling the original PD points in red) along with $\delta$-bands for method 1, 2, and 3 (in pink, black, and magenta dotted lines). We observe that method 1 mostly has a very high $\delta$-band, such that it considers all the PD points as topological noise, except for a very few cases (as shown in the first column). Additionally, in most cases, \textit{TopoGAT} and method 3 perform similarly to each other, while method 2 generally has a $\delta$-band higher than that of method 3 and hence considers topological features as topological noise that are recognized as significant by method 3 as well as \textit{TopoGAT}. A few failed cases of method 3 are shown in rows 1, 2, and 3 in column 4, where method 3 could not recognize a significant feature.

A key limitation of methods 2 and 3 is that they derive a single threshold value from each point cloud to filter out insignificant topological features in both persistence diagrams $H_{1}$ and $H_{2}$. As a result, these methods do not account for the distinct distributions of individual homology classes. 
\begin{figure}[h]
\centering
\includegraphics[width=8.0 cm]{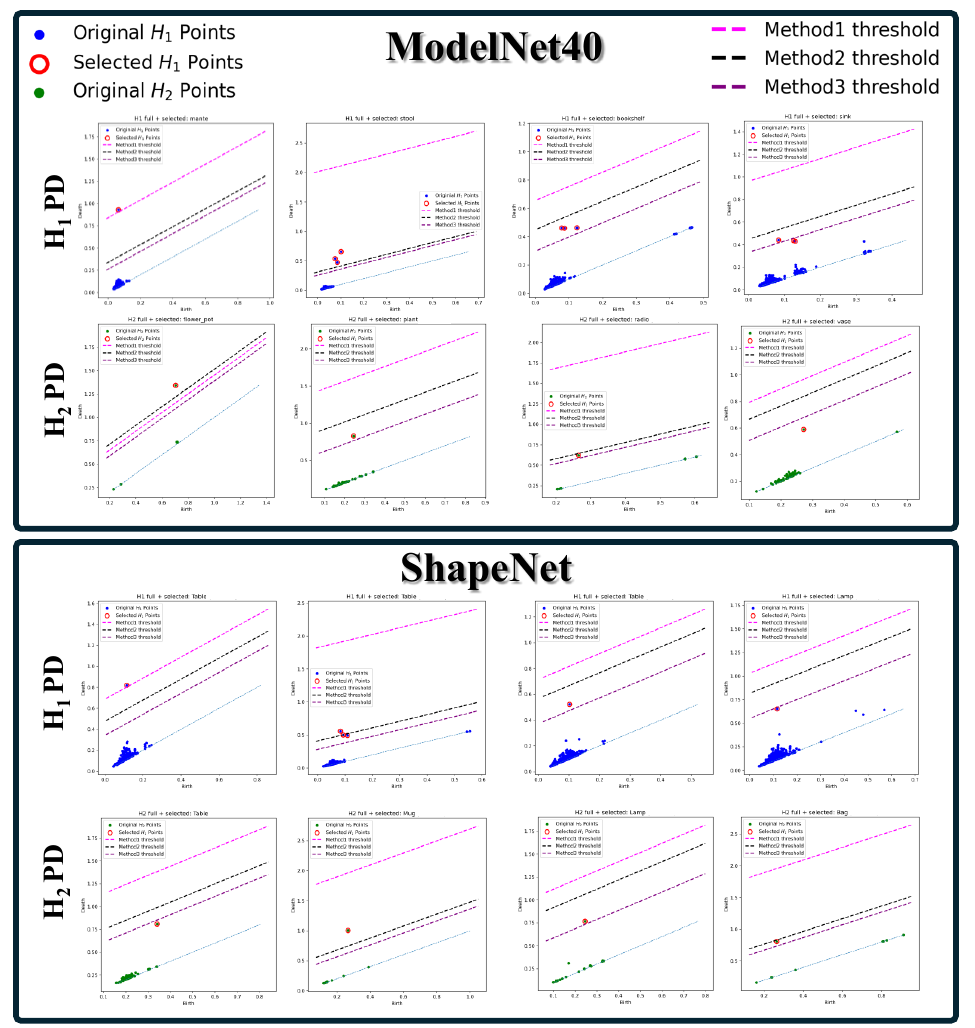}
\caption{Qualitative comparison of significant topological feature selection by various statistical methods and proposed \textit{TopoGAT}. Blue and green denote original $H_{1}$ and $H_{2}$ PD points. The red circles encircling original PD points denote the \textit{TopoGAT}-selected PD points (Selected $H_{1}$ or $H_{2}$ points). Pink, black, and magenta dotted lines denote $\delta$-bands for methods 1, 2, and 3 respectively.}
\label{pdcomp}
\end{figure}

\begin{table}[h!]
\centering
\captionsetup{singlelinecheck=false, justification=centering}
\caption{Classification results on ModelNet40 dataset for different GNN architectures with different input combinations, using PD based vectorization. OA and mAcc in \% . W/A and W/S expand to with all and with selected respectively. }
\small
\resizebox{\linewidth}{!}{%
\begin{tabular}{|l|cc|cc|cc|}
\hline
\textbf{Input(s)} &
\multicolumn{2}{|c|}{\textbf{GCN}} &
\multicolumn{2}{c|}{\textbf{GAT}} &
\multicolumn{2}{c|}{\textbf{GIN}} \\ \cline{2-7}
 & \textbf{OA} & \textbf{mAcc} 
 & \textbf{OA} & \textbf{mAcc}
 & \textbf{OA} & \textbf{mAcc} \\
\hline
Baseline  & 80.91 & 76.89 & 82.18 & 78.98 & 68.23 & 56.24 \\
W/A PD points & 82.01 & 79.66 & 83.77 & 80.76 & 69.31 & 58.46 \\
W/S PD points & 84.98 & 80.85 & \textbf{87.14} & 82.01 & 69.87 & 56.01 \\
\hline
\end{tabular}
}
\label{gcn_gat_gin_results}
\end{table}
\begin{table}[h!]
\centering
\captionsetup{singlelinecheck=false, justification=centering}
\caption{Classification results on ModelNet40 dataset for different vectorization with different input combinations. The DL architecture used is indicated in parentheses. OA and mAcc in \% . W/A and W/S expand to with all and with selected respectively.}
\small
\resizebox{\linewidth}{!}{%
\begin{tabular}{|l|cc|cc|cc|}
\hline
\textbf{Input(s)} &
\multicolumn{2}{c|}{\textbf{PL (GIN)}} &
\multicolumn{2}{c|}{\textbf{PI (ResNet\cite{torchvision2016})}} &
\multicolumn{2}{c|}{\textbf{PD (GAT)}} \\ \cline{2-7}
 & \textbf{OA} & \textbf{mAcc} 
 & \textbf{OA} & \textbf{mAcc}
 & \textbf{OA} & \textbf{mAcc} \\
\hline
Baseline  & 82.49 & 78.27 & 78.68 & 73.34 & 82.18 & 78.98 \\
W/A PD points & 82.78 & 78.65 & 80.25 & 72.81 & 83.77 & 80.76 \\
W/S PD points & 83.05 & 79.04 & 80.93 & 76.97 & \textbf{87.14} & 82.01 \\
\hline
\end{tabular}
}
\label{vectorization_results}
\end{table}
\begin{table*}[t]
\centering
\captionsetup{singlelinecheck=false, justification=centering}
\captionof{table}{Part segmentation Results on ShapeNet: mean IoU, point-wise accuracy, and per-category IoU (\%). Significant PD points selected by \textit{TopoGAT} are abbreviated as S:H1-PD and S:H2-PD. Second row corresponds to the number of \# of point clouds (PC) in each category. }
\footnotesize
\setlength{\tabcolsep}{1pt}
\resizebox{\linewidth}{!}{%
\begin{tabular}{|p{1.0cm}|c|c|c|cccccccccccccccc|}
\hline
\textbf{Archi.} & \textbf{Input(s)} & \textbf{mIoU} & \textbf{pAcc} &
\textbf{Airplane} & \textbf{Bag} & \textbf{Cap} & \textbf{Car} &
\textbf{Chair} & \textbf{Earph.} & \textbf{Guitar} & \textbf{Knife} &
\textbf{Lamp} & \textbf{Laptop} & \textbf{Motor} & \textbf{Mug} &
\textbf{Pistol} & \textbf{Rocket} & \textbf{Skate} & \textbf{Table} \\
\hline
\multicolumn{4}{|c|}{} &
341 & 14 & 12 & 158 & 704 & 14 & 159 & 83 & 288 & 83 & 51 & 38 & 44 & 12 & 31 & 848 \\
\hline

 & PC & 66.7 & 90.6 &
68.6 & 64.4 & 65.9 & 64.8 &
68.4 & 63.9 & 68.1 & 67.2 &
65.3 & 67.5 & 58.6 & 68.1 &
63.8 & 59.3 & 62.7 & 66.1 \\

GAT & PC + H1-PD + H2-PD & 67.6 & 90.4 &
69.3 & 66.5 & 65.3 & 66.7 &
68.7 & 65.6 & 68.3 & 69.3 &
66.7 & 68.6 & 60.9 & 69.8 &
65.7 & 58.8 & 63.1 & 67.1 \\

 & PC + S:H1-PD + S:H2-PD & \textbf{70.6} & \textbf{91.9} &
70.9 & 68.3 & 67.8 & 69.9 &
72.2 & 68.6 & 73.9 & 71.1 &
69.7 & 72.3 & 61.7 & 71.6 &
68.7 & 60.1 & 65.7 & 69.8 \\
\hline

 & PC & 61.5 & 88.7 &
63.7 & 60.4 & 61.8 & 58.9 &
63.2 & 59.7 & 63.4 & 61.9 &
60.1 & 61.9 & 54.8 & 64.4 &
59.9 & 54.8 & 57.9 & 60.6 \\

GCN & PC + H1-PD + H2-PD & 62.2 & 88.6 &
63.3 & 60.9 & 61.9 & 61.4 &
63.3 & 60.8 & 62.4 & 63.1 &
62.8 & 63.8 & 57.8 & 65.4 &
61.3 & 54.9 & 59.1 & 61.2 \\

 & PC + S:H1-PD + S:H2-PD & \textbf{64.3} & \textbf{89.9} &
65.1 & 63.1 & 63.2 & 63.5 &
65.8 & 62.3 & 64.5 & 66.7 &
63.2 & 66.3 & 59.9 & 68.2 &
61.7 & 57.1 & 59.9 & 63.4 \\
\hline
\end{tabular}
}
\label{segmentation_results}
\end{table*}
\subsection{Downstream Task: Classification}\label{classification}

To understand the performance of the \textit{TopoGAT}-based significant point selection for point cloud classification on different GNN architectures, we use three different classification networks, each based on Graph Convolution Network (GCN), Graph Attention Network (GAT), and Graph Isomorphism Network (GIN). The experimental results of this comparative study are shown in Table \ref{gcn_gat_gin_results}. To keep these architectures as close as possible to each other, we just replaced the GNN layer in the whole architecture while keeping all other neural network components and parameters the same. We observe that in each case, the input PC with significant PD points (W/S PD points) shows higher overall accuracy as well as mean class-wise accuracy. 

Similarly, Table \ref{vectorization_results} shows the classification accuracies for different vectorization methods, such as persistence image (PI) and persistence landscape (PL), for the topological representation of the point cloud. Similar to the previous results, we observe that in addition to the geometric input i.e. $(x,y,z)$ coordinates of a point cloud, additional topological inputs $H_{1}$ PDs and $H_{2}$ PDs increase the classification accuracy, which is further enhanced in the case of using \textit{TopoGAT}-selected significant $H_{1}$ PDs and $H_{2}$ PD points as input. These results show that topological information provides complementary discriminative power beyond geometry alone, and that \textit{TopoGAT} effectively identifies and selects the most informative topological features for classification.

\subsection{Downstream Task: Part Segmentation}
We perform part-segmentation of ShapeNet point clouds to understand the performance of the \textit{TopoGAT}-based significant point selection. As shown in Table \ref{segmentation_results}, this experiment includes three kinds of GNN part-segmentation pipelines. The first architecture consists only of a point cloud as the input to the neural network, and hence it can be considered as the baseline architecture. The second architecture consists of three branches having an input as a point cloud, its corresponding full $H_{1}$-PD, and full $H_{2}$-PD. Similar to the second architecture, the third pipeline also consists of three branches with point cloud, $H_{1}$-PD, and $H_{2}$-PD; however, both the PDs consist of only the significant persistent points selected by the \textit{TopoGAT}. Additionally, we use GAT and GCN as two different GNNs to validate the performance of the \textit{TopoGAT}-based feature selection for different GNN architectures. We observe that in both GCN as well as in GAT architectures, the inclusion of topological features increases the mean intersection-over-union (mIoU) of the part-segmentation. However, instead of using full PD, the topological features (PDs) filtered by \textit{TopoGAT} further enhance the mIoU. This demonstrates that excluding insignificant topological features leads to more discriminative representations and improved segmentation performance.

Figure \ref{segresults} presents the part-segmentation visualizations on ShapeNet objects using both GAT and GCN architectures. For both of the GNN architectures, when relying solely on geometric input (PC), the models struggle to correctly identify distinct object parts such as those of the \textit{cap}, \textit{rocket}, \textit{car}, and \textit{earphone}. Incorporating topological input leads to noticeably improved segmentation quality, and the performance is further enhanced when using the \textit{TopoGAT}-selected PD features. These results demonstrate the effectiveness of topological information in point-cloud learning.

\begin{figure}[h]
\centering
\includegraphics[width=8.0 cm]{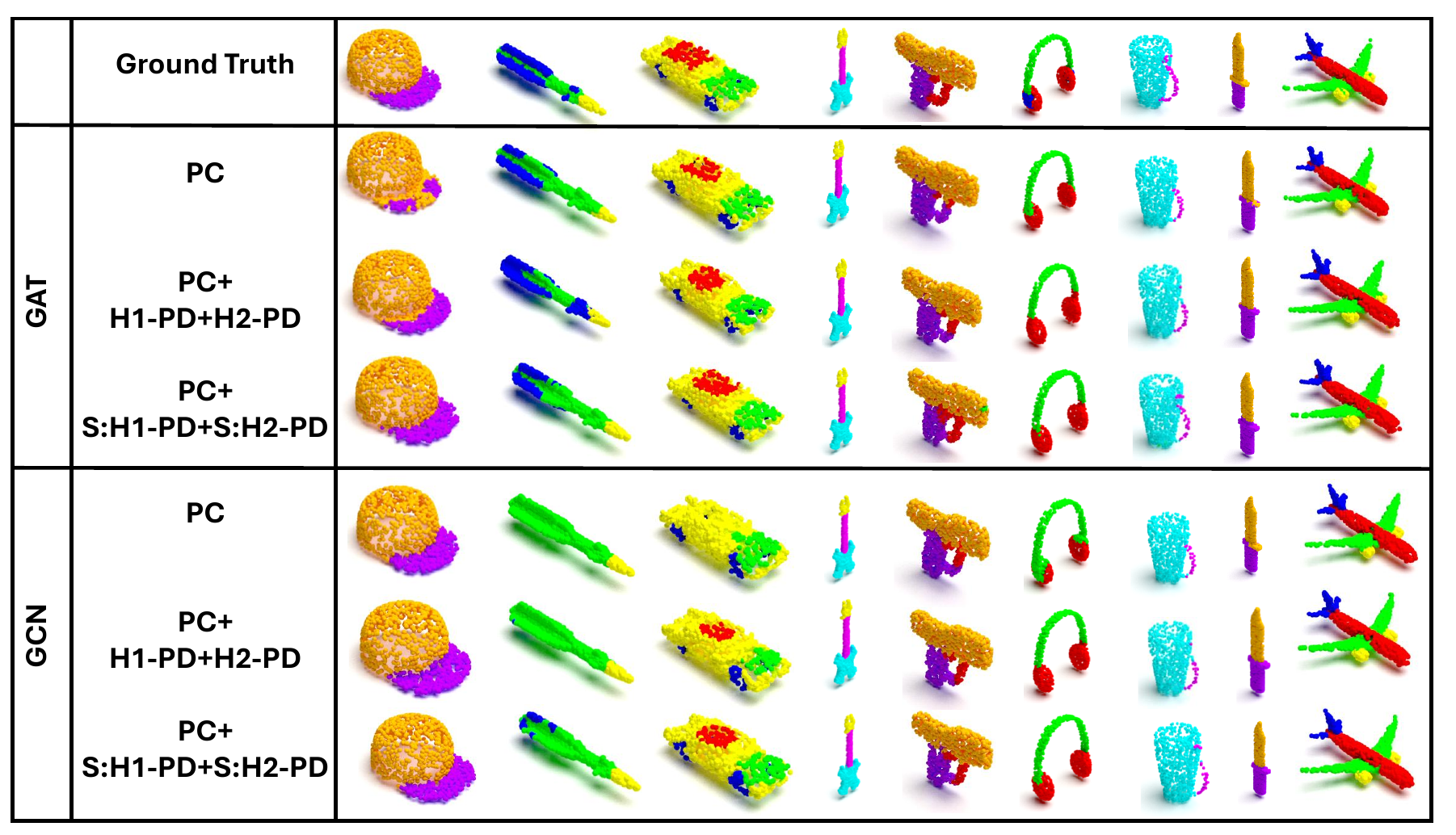}
\caption{Qualitative results for part segmentation on the eight ShapeNet objects for different inputs in GAT and GCN architectures. Significant PD points selected by \textit{TopoGAT} are abbreviated as S:H1-PD and S:H2-PD.}
\label{segresults}
\end{figure}
\subsection{Ablation on Loss Functions}
To understand the effect of different components of the proposed topological loss on downstream tasks, we use three variants of this topological loss. One loss variant \textit{WER} is the same as the proposed topological loss, consisting of all three components: Wasserstein Distance Loss, Persistent Entropy Loss, and Reduction Loss. The second variant \textit{WE} consists of two loss components, which are Wasserstein Distance Loss and Persistent Entropy Loss, such that each has a learnable coefficient. In the third variant \textit{W}, we only consider Wasserstein Distance Loss with a learnable coefficient. Using three different topological loss functions, we filter significant topological features using \textit{TopoGAT}, which are further used to perform downstream tasks classification on ModelNet40, shown in Table \ref{loss_comparison_cls}, and Part-segmentation on ShapeNet, shown in Table \ref{loss_comparison_seg}. 

To understand the effect of the loss function for different GNNs, we use three different GNN architectures, each with GCN, GAT, and GIN, respectively, for classification. As GIN was originally designed for graph-level classification, we do not consider it for segmentation experiments, which require per-point prediction. Thus, to study the effect of different TopoLoss components, for part-segmentation experiments, we only use GCN and GAT architectures. As shown in Table \ref{loss_comparison_cls} and \ref{loss_comparison_seg}, we observe that in classification as well as part-segmentation, the topological input with significant topological points selected using TopoLoss with all three components outperforms its two other variants \textit{WE} and \textit{W} in all types of GNN architectures. This shows that the inclusion of persistent entropy loss and reduction loss benefits the downstream tasks by selecting more significant topological points. 

\begin{table}[h!]
\centering
\captionsetup{singlelinecheck=false, justification=centering}
\caption{Comparison of Loss Functions for GAT, GCN, and GIN-based architectures for classification on ModelNet40. OA and mAcc in \%. W, E and R stand for Wasserstein Distance Loss, Persistent Entropy Loss, and
Reduction Loss, respectively.}
\small
\resizebox{\linewidth}{!}{%
\begin{tabular}{|p{1.9cm}|cc|cc|cc|}
\hline
\textbf{Loss Function Variant} &
\multicolumn{2}{|c|}{\textbf{GAT}} &
\multicolumn{2}{c|}{\textbf{GCN}} &
\multicolumn{2}{c|}{\textbf{GIN}}\\ \cline{2-7}
 & \textbf{OA} & \textbf{mAcc} 
 & \textbf{OA} & \textbf{mAcc} 
& \textbf{OA} & \textbf{mAcc} \\
\hline
WER & 87.14 & 82.01 & 84.98 & 80.85 & 69.87 & 56.01\\
WE  & 84.68 & 80.36 & 82.04 & 76.53 & 65.75 & 53.97 \\
W   & 82.97 & 79.43 & 81.70 & 76.28 & 64.82 & 53.43 \\
\hline
\end{tabular}
}
\label{loss_comparison_cls}
\end{table}

\begin{table}[h!]
\centering
\captionsetup{singlelinecheck=false, justification=centering}
\caption{Comparison of Loss Functions for GAT and GCN-based architectures for Part-segmentation on ShapeNet. mIou and pAcc in \%. W, E and R stand for Wasserstein Distance Loss, Persistent Entropy Loss, and
Reduction Loss, respectively.}
\small
\begin{tabular}{|p{1.9cm}|cc|cc|}
\hline
\textbf{Loss Function Variant} &
\multicolumn{2}{|c|}{\textbf{GAT}} &
\multicolumn{2}{c|}{\textbf{GCN}} \\ \cline{2-5}
 & \textbf{mIoU} & \textbf{pAcc} 
 & \textbf{mIoU} & \textbf{pAcc} \\
\hline
WER & 70.6 & 91.9 & 64.3 & 89.9 \\
WE  & 67.1 & 90.6 & 62.8 & 88.5 \\
W   & 66.6 & 90.3 & 61.6 & 88.3 \\
\hline
\end{tabular}
\label{loss_comparison_seg}
\end{table}

\subsection{Limitations}

Through the comparative study of the downstream tasks classification and part-segmentation, we aim to demonstrate the necessity of incorporating topological features and to evaluate the effectiveness of the proposed \textit{TopoGAT}-based significant topological feature selection strategy. While the present analysis is restricted to fundamental graph neural network (GNN) architectures, future work may explore its applicability to more sophisticated neural network architectures.

For the downstream tasks, we employ standard geometric loss functions, that is, cross-entropy, to facilitate learning from both raw point clouds and persistence diagram-derived feature maps. However, we believe that integrating topologically informed loss components could offer a more explicit encoding of topological structure, potentially enhancing performance on downstream objectives.

\section{Conclusion and Future Work}
\label{conclusion}
This paper presents the topological datasets of benchmark 3D point cloud datasets that provide topological representations of each point cloud in homology dimensions 1 and 2. These datasets can be directly utilized for persistent homology-based point cloud learning. A novel GAT-based topological feature selection neural network, \textit{TopoGAT}, presents a topologically efficient solution for significant feature selection of large point cloud datasets by utilizing a novel topological loss term, TopoLoss. Comparative study with the statistical methods shows the efficient selection of significant features without significantly altering the topology of the dataset. Further, the comparative study on various GNN networks with topological features showcases the need for further exploration of topological information in 3D shape analysis.

In this work, we use a classification-guided network for significant topological feature selection, and our results show the effectiveness of its feature selection in both classification as well as part-segmentation. However, we believe that \textit{TopoGAT} can be extended to any supervised learning network for point cloud learning based downstream tasks as a preprocessing module, which can guide the \textit{TopoGAT} learning based on specific tasks in the future. Moreover, we also believe that the inclusion of topological loss terms in the loss function for downstream tasks would further enhance the ability of task-specific models to leverage topological information, and hence, it remains a promising direction for future work.


\bibliographystyle{cag-num-names}
\bibliography{refs}

@article{cohen2007stability,
  author    = {Cohen-Steiner, David and Edelsbrunner, Herbert and Harer, John},
  title     = {Stability of Persistence Diagrams},
  journal   = {Mathematics of Computation},
  volume    = {76},
  number    = {257},
  pages     = {1811--1833},
  year      = {2007},
  doi       = {10.1090/S0025-5718-07-01970-6}
}

@article{ctralie2018ripser,
    doi = {10.21105/joss.00925},
    url = {https://doi.org/10.21105/joss.00925},
    year  = {2018},
    month = {Sep},
    publisher = {The Open Journal},
    volume = {3},
    number = {29},
    pages = {925},
    author = {Christopher Tralie and Nathaniel Saul and Rann Bar-On},
    title = {{Ripser.py}: A Lean Persistent Homology Library for Python},
    journal = {The Journal of Open Source Software}
}

@article{fasy2014confidence,
  title={Confidence sets for persistence diagrams},
  author={Fasy, Brittany Terese and Lecci, Fabrizio and Rinaldo, Alessandro and Wasserman, Larry and Balakrishnan, Sivaraman and Singh, Aarti},
  year={2014},
  journal={arxiv}
}

@article{edelsbrunner2002topological,
  title={Topological persistence and simplification},
  author={Edelsbrunner and Letscher and Zomorodian},
  journal={Discrete \& computational geometry},
  volume={28},
  number={4},
  pages={511--533},
  year={2002},
  publisher={Springer}
}

@article{clough2020topological,
  title={A topological loss function for deep-learning based image segmentation using persistent homology},
  author={Clough, James R and Byrne, Nicholas and Oksuz, Ilkay and Zimmer, Veronika A and Schnabel, Julia A and King, Andrew P},
  journal={IEEE transactions on pattern analysis and machine intelligence},
  volume={44},
  number={12},
  pages={8766--8778},
  year={2020},
  publisher={IEEE}
}

@book{dey2022computational,
  title={Computational topology for data analysis},
  author={Dey, Tamal Krishna and Wang, Yusu},
  year={2022},
  publisher={Cambridge University Press}
}

@inproceedings{wu20153d,
  title={3d shapenets: A deep representation for volumetric shapes},
  author={Wu, Zhirong and Song, Shuran and Khosla, Aditya and Yu, Fisher and Zhang, Linguang and Tang, Xiaoou and Xiao, Jianxiong},
  booktitle={Proceedings of the IEEE conference on computer vision and pattern recognition},
  pages={1912--1920},
  year={2015}
}

@inproceedings{liu2022toposeg,
  title={TopoSeg: Topology-aware Segmentation for Point Clouds.},
  author={Liu, Weiquan and Guo, Hanyun and Zhang, Weini and Zang, Yu and Wang, Cheng and Li, Jonathan},
  booktitle={IJCAI},
  pages={1201--1208},
  year={2022}
}

@article{montufar2020can,
  title={Can neural networks learn persistent homology features?},
  author={Mont{\'u}far, Guido and Otter, Nina and Wang, Yuguang},
  journal={arXiv preprint arXiv:2011.14688},
  year={2020}
}

@inproceedings{peek2023synthetic,
  title={Synthetic data generation and deep learning for the topological analysis of 3d data},
  author={Peek, Dylan and Skerritt, Matthew P and Chalup, Stephan},
  booktitle={2023 International Conference on Digital Image Computing: Techniques and Applications (DICTA)},
  pages={121--128},
  year={2023},
  organization={IEEE}
}

@article{zhou2022learning,
  title={Learning persistent homology of 3D point clouds},
  author={Zhou, Chi and Dong, Zhetong and Lin, Hongwei},
  journal={Computers \& Graphics},
  volume={102},
  pages={269--279},
  year={2022},
  publisher={Elsevier}
}

@article{hu2024topology,
  title={Topology-aware latent diffusion for 3d shape generation},
  author={Hu, Jiangbei and Fei, Ben and Xu, Baixin and Hou, Fei and Yang, Weidong and Wang, Shengfa and Lei, Na and Qian, Chen and He, Ying},
  journal={arXiv preprint arXiv:2401.17603},
  year={2024}
}

@article{jignasu2024stitch,
  title={STITCH: Surface reconstrucTion using Implicit neural representations with Topology Constraints and persistent Homology},
  author={Jignasu, Anushrut and Herron, Ethan and Jiang, Zhanhong and Sarkar, Soumik and Hegde, Chinmay and Ganapathysubramanian, Baskar and Balu, Aditya and Krishnamurthy, Adarsh},
  journal={arXiv preprint arXiv:2412.18696},
  year={2024}
}

@article{zia2024topological,
  title={Topological deep learning: a review of an emerging paradigm},
  author={Zia, Ali and Khamis, Abdelwahed and Nichols, James and Tayab, Usman Bashir and Hayder, Zeeshan and Rolland, Vivien and Stone, Eric and Petersson, Lars},
  journal={Artificial Intelligence Review},
  volume={57},
  number={4},
  pages={77},
  year={2024},
  publisher={Springer}
}

@article{malott2022survey,
  title={A survey on the high-performance computation of persistent homology},
  author={Malott, Nicholas O and Chen, Shangye and Wilsey, Philip A},
  journal={IEEE Transactions on Knowledge and Data Engineering},
  volume={35},
  number={5},
  pages={4466--4484},
  year={2022},
  publisher={IEEE}
}

@article{coskunuzer2024topological,
  title={Topological methods in machine learning: A tutorial for practitioners},
  author={Coskunuzer, Baris and Ak{\c{c}}ora, C{\"u}neyt G{\"u}rcan},
  journal={arXiv preprint arXiv:2409.02901},
  year={2024}
}

@inproceedings{qi2017pointnet,
  title={Pointnet: Deep learning on point sets for 3d classification and segmentation},
  author={Qi, Charles R and Su, Hao and Mo, Kaichun and Guibas, Leonidas J},
  booktitle={Proceedings of the IEEE conference on computer vision and pattern recognition},
  pages={652--660},
  year={2017}
}

@inproceedings{fugacci2016persistent,
  title={Persistent Homology: a Step-by-step Introduction for Newcomers.},
  author={Fugacci, Ulderico and Scaramuccia, Sara and Iuricich, Federico and De Floriani, Leila and others},
  booktitle={STAG},
  pages={1--10},
  year={2016}
}

@article{berwald2018computing,
  title={Computing Wasserstein distance for persistence diagrams on a quantum computer},
  author={Berwald, Jesse J and Gottlieb, Joel M and Munch, Elizabeth},
  journal={arXiv preprint arXiv:1809.06433},
  year={2018}
}

@article{agami2023comparison,
  title={Comparison of persistence diagrams},
  author={Agami, Sarit},
  journal={Communications in Statistics-Simulation and Computation},
  volume={52},
  number={5},
  pages={1948--1961},
  year={2023},
  publisher={Taylor \& Francis}
}

@article{merelli2015topological,
  title={Topological characterization of complex systems: Using persistent entropy},
  author={Merelli, Emanuela and Rucco, Matteo and Sloot, Peter and Tesei, Luca},
  journal={Entropy},
  volume={17},
  number={10},
  pages={6872--6892},
  year={2015},
  publisher={MDPI}
}

@misc{torchvision2016,
    title        = {TorchVision: PyTorch's Computer Vision library},
    author       = {TorchVision maintainers and contributors},
    year         = 2016,
    journal      = {GitHub repository},
    publisher    = {GitHub},
    howpublished = {\url{https://github.com/pytorch/vision}}
}

@article{shapenet2015,
  title={A scalable active framework for region annotation in 3d shape collections},
  author={Yi, Li and Kim, Vladimir G and Ceylan, Duygu and Shen, I-Chao and Yan, Mengyan and Su, Hao and Lu, Cewu and Huang, Qixing and Sheffer, Alla and Guibas, Leonidas},
  journal={ACM Transactions on Graphics (ToG)},
  volume={35},
  number={6},
  pages={1--12},
  year={2016},
  publisher={ACM New York, NY, USA}
}

\end{document}